\pdfoutput=1
\documentclass[12pt]{amsart}
\usepackage{mathrsfs}
\usepackage{soul}
\usepackage{fix-cm} 
\usepackage[dvipsnames]{xcolor}
\usepackage{hyperref}
\usepackage[numbers, authoryear, round]{natbib}
\usepackage[margin=1in, left=1in, right=1in]{geometry}
\usepackage{setspace}
\usepackage{fancyhdr}

\usepackage{amsmath, amssymb, amsthm}
\hypersetup{
    colorlinks,
    citecolor=blue,
    urlcolor=blue,
    linkcolor=blue,
    linktocpage=true
}
\newcommand{\myRN}[1]{\textup{\uppercase\expandafter{\romannumeral#1}}}

\def\argmax{\mathop{\rm arg\, max}}
\def\argmin{\mathop{\rm arg\, min}}
\def\bbar{\overline{b}}
\def\cbar{\overline{c}}
\def\gbar{\overline{g}}
\def\xbar{\overline{x}}
\def\ybar{\overline{y}}

\newcommand{\ind}{\mathbold{1}}
\def\vepsbar{{\overline\veps}}
\def\UCB{\hbox{\rm UCB}}
\def\KL{\hbox{\rm KL}}

\def\argmax{\mathop{\rm arg\, max}}
\def\argmin{\mathop{\rm arg\, min}}

\newcommand{\bel}{\begin{eqnarray}\label}
\newcommand{\eel}{\end{eqnarray}}
\newcommand{\bes}{\begin{eqnarray*}}
\newcommand{\ees}{\end{eqnarray*}}
\newcommand{\bei}{\begin{itemize}}
\newcommand{\eei}{\end{itemize}}
\newcommand{\beiftnt}{\begin{itemize}\footnotesize}

\def\benu{\begin{enumerate}}
\def\eenu{\end{enumerate}}

\def\argmax{\mathop{\rm arg\, max}}
\def\argmin{\mathop{\rm arg\, min}}

\def\R{{\real}}

\def\E{{\mathbb{E}}}

\def\P{{\mathbb{P}}}

\def\R{{\mathbb{R}}}

\def\complex{\mathop{{\rm I}\kern-.58em\hbox{\rm C}}\nolimits}

\def\f{\frac}

\def\mathbold{\boldsymbol} 

\def\scrF{{\mathscr F}}

\def\xbar{{\overline x}}

\def\ybar{{\overline y}}

\def\eps{\epsilon}\def\veps{\varepsilon}

\def\thetahat{\widehat{\theta}}

\def\muhat{\widehat{\mu}}

\newtheorem{theorem}{Theorem}
\newtheorem{lemma}[theorem]{Lemma} 
 
\newtheorem{remark}[theorem]{Remark}
\newtheorem{corollary}[theorem]{Corollary}

\begin{document}
\title{On Lai's Upper Confidence Bound in Multi-Armed Bandits }
\author{Huachen Ren}
\author{Cun-Hui Zhang}
\address{Department of Statistics, Rutgers University\\
       110 Frelinghuysen Road\\
       Piscataway, New Jersey 08854, USA}
\email{huachen.ren@rutgers.edu}
\email{czhang@stat.rutgers.edu}
\thanks{Zhang’s research is supported in part by National Science Foundation grants
CCF-1934924, DMS-2052949 and DMS-2210850.}

\begin{abstract}
In this memorial paper, we honor Tze Leung Lai's seminal contributions to the topic of multi-armed bandits, with a specific focus on his pioneering work on the upper confidence bound. We establish sharp non-asymptotic regret bounds for an upper confidence bound index with a constant level of exploration for Gaussian rewards. Furthermore, we establish a non-asymptotic regret bound for the upper confidence bound index of \cite{lai1987adaptive} which employs an exploration function that decreases with the sample size of the corresponding arm. The regret bounds have leading constants that match the Lai-Robbins lower bound. Our results highlight an aspect of Lai's seminal works that deserves more attention in the machine learning literature.  
\end{abstract}
\maketitle

\section{Introduction}
Originating from Thompson's seminal work \citep{thompson1933likelihood} on clinical trials, the multi-armed bandit problem was formally introduced and popularised by \cite{robbins1952sequential}, evolving into a cornerstone of sequential decision--making in both statistics and machine learning. The multi-armed bandit problem concerns $K$ populations (arms) and the choice of adaptive allocation rules $\phi_1, \phi_2, \ldots$ taking values in $\{1, \ldots, K\}$. An agent selects arm $a$ at time $t$ if $\phi_t = a$, and subsequently receives a reward $y_t$ from the chosen arm. An allocation rule is adaptive if $\phi_t$  depends only on the previous allocations and rewards $\phi_1, y_1, \ldots, \phi_{t-1}, y_{t-1}$. 
Adaptive allocation rules are often referred to as policies or algorithms in the machine learning literature.  
The objective of the agent is to maximize the expected cumulative reward up to a time horizon $T$. 
It follows from the optional stopping theorem that if the identity were known for a population with maximum mean, the agent would be able to maximize the expected reward by sampling exclusively the optimal arm.  
Without knowing the optimal arm, a balance must be struck between exploring various arms to estimate their mean rewards and exploiting the most promising arm based on current information. This dilemma, known as the exploration--exploitation trade-off, is a common challenge in reinforcement learning, and more generally in sequential design of statistical experiments. 

Significant research in multi-armed bandits focused on the study of Bayesian optimal policies from 1960 to 1980, as explored in the seminal papers by \cite{bellman1956problem} and \cite{bradt1956sequential}.
A notable breakthrough was the introduction of the Gittins index in \cite{gittins1979dynamic} and \cite{gittins1979bandit}, providing the optimal Bayesian strategy in the setting of infinite--horizon discounted rewards. At each time point, Gittins' policy computes an index for each arm that depends solely on the observed samples of that arm, and selects the arm with the highest index. Such policies, referred to as index policies in the literature, are highly attractive as they are typically easy to explain.

In the frequentist framework formulated by \cite{robbins1952sequential}, 
the regret of an allocation rule, defined as 
\begin{align*}
    R_{T}= T\mu^* - \E\bigg[\sum_{t=1}^Ty_t\bigg],
\end{align*}
is commonly used to measure its performance, 
where $\mu^*$ is the mean of the optimal arm. For $K=2$, \cite{robbins1952sequential} proposed an allocation rule which achieves $R_T = o(T)$. Although Robbin's procedure implies that the average regret $R_T/T$ converges to zero, an optimal allocation rule with asymptotically the smallest regret remained unknown until \cite{lai1985asymptotically} 
established an information lower bound for the regret and proposed an asymptotically optimal allocation rule to achieve the lower bound 
in their groundbreaking work. 
\subsection{The Lai--Robbbins lower bound}
\cite{lai1985asymptotically} established the first frequentist asymptotic lower bound of the regret for bandits with parametric reward distributions. The lower bound was subsequently generalized to multi-armed bandits with multi-parameter and nonparametric rewards \citep{burnetas1996optimal}, controlled Markov chains \citep{graves1997asymptotically} and reinforcement learning \citep{burnetas1997optimal}. 

Assume that each arm has a density function $f_{\theta}$ as a member of a parametric family of distributions with unknown parameter $\theta$. Under mild regularity conditions on $f_{\theta}$, 
\cite{lai1985asymptotically} proved that for any ``consistent" allocation rule 
satisfying $R_T = o(T^{p})$ for any $p>0$, the following information lower bound must hold for the regret,
\begin{align*}
\liminf_{T \rightarrow \infty}\f{R_T }{\log T} \ge \sum_{a: \mu_a < \mu^*}\f{\mu^* - \mu_a}{\KL(\theta_a, \theta^*)},
\end{align*}
where $\theta^*$ is the parameter of the optimal arm, $\mu_a$ is the mean of arm $a$, and $\KL(\theta_a, \theta^*)$ is the Kullback-Leibler (KL) divergence between $f_{\theta_a}$ and $f_{\theta^*}$. 
This Lai-Robbins lower bound characterizes the overall information complexity of the bandit instance $\{f_{\theta_a}\}_{a=1}^K$, demonstrating that any consistent 
allocation rule achieving the lower bound must sample each inferior arm $a$ at least $\log 
 (T)/\KL(\theta_a, \theta^*)$ times asymptotically. 

 Another notable contribution of \cite{lai1985asymptotically} is the introduction of the concept of upper confidence bound (UCB), along with an allocation rule that asymptotically attains the lower bound. For each arm $a$, their procedure cyclically compares the UCB of arm $a$ with the sample mean of the ``leading'' arm. 
 When arm $a$ reaches its turn for possible allocation, 
it is sampled if its UCB exceeds the sample mean of the leading arm, and the leading arm is sampled otherwise. Due to the cyclic structure of the procedure, their policy is not an index policy. 
Later, \cite{lai1987adaptive} proposed an index policy based on UCB for a predetermined horizon $T$.  \cite{agrawal1995sample} and \cite{katehakis1995sequential} developed and studied UCB indices in the ``anytime" setting where the agent's performance is measured continuously without a predetermined horizon, 
respectively for exponential family rewards and Gaussian rewards. \cite{burnetas1996optimal} generalized UCB to multi-parameter and nonparametric reward distributions. 

\subsection{Lai's UCB}
\cite{lai1987adaptive} introduced the first UCB index policy for multi-armed bandits. 
Consider a $K$-armed bandit problem with reward distributions in a one-parameter exponential family. 
Let $\thetahat_{a, t}$ be the maximum likelihood estimator of the parameter $\theta_a$ of arm $a$ based on data available at time $t$, $n_{a,t}$ be the sample size of arm $a$ at time $t$, and $\KL(\theta_a, \theta_{a'})$ be the KL divergence between the reward distributions of arms $a$ and $a'$. 
After sampling each arm once, Lai's index is defined as 
\begin{align}\label{lai-index}
	\UCB_{a, t}^{\text{Lai}} =  \inf\big\{\theta: \theta>\thetahat_{a, t-1}, \KL(\thetahat_{a, t-1}, \theta) \ge g(T/n_{a,t-1})/n_{a,t-1} \big\},
\end{align}
where $g(x)$ can be any function satisfying (i) $\sup_{1 \le x \le t}xg(x)< \infty$ for any $t \ge 1$, (ii) $g(t) \sim \log t$, and (iii) $g(t) \ge \log t+\xi \log\log t$ as $t \rightarrow \infty$ for some $\xi>-3/2$. The function $g(T/n_{a,t-1})$ controls the margin error and is referred to as the exploration function in the machine learning literature \citep{audibert2009exploration}. Lai proved that his UCB index \eqref{lai-index} 
achieves the asymptotic lower bound of \cite{lai1985asymptotically}, and also approximates the Bayesian optimal policy asymptotically under mild conditions on the prior. Lai's analysis was based on his work on boundary crossing probabilities \citep{lai1988boundary}. In an accompanying paper, \cite{chang1987optimal}  showed that the Gittins index could also be approximated by an index of a similar form to Lai's index \eqref{lai-index} in the setting of infinite--horizon discounted rewards. 

In modern machine learning, variants of Lai's UCB were developed by inverting the KL divergence as in \eqref{lai-index} with various exploration functions for predetermined or unspecified horizon. For bandits with reward distributions in one-parameter exponential family, 
\cite{garivier2011kl,cappe2013kl} called the following index kl-UCB, 
\begin{align}\label{cape-kl-ucb}
	\hbox{\rm kl-UCB}_{a,t}= \sup \Big\{\theta\ge \thetahat_{a,t-1}: \KL(\thetahat_{a,t-1}, \theta) \le f(t)/n_{a,t-1} \Big\},\ \forall t > K,
\end{align}
where $f(t) = \log t+3\log\log t$. They established a non-asymptotic regret bound whose leading constant achieves the Lai-Robbins lower bound and generalized the result to bounded rewards with finite support. 
We notice that Lai's UCB in \eqref{lai-index} uses an exploration function that decreases with the sample size of the corresponding arm. \cite{garivier2011kl} called the index kl-UCB+ when the $f(t)$ in \eqref{cape-kl-ucb} is replaced by $\log(t/n_{a,t-1})$, and studied its performance empirically. The idea of tuning the exploration function based on the sample size also appeared in \cite{audibert2009minimax} who developed a UCB index called MOSS, which 
replaces $g(T/n_{a,t-1})$ in (1) by $\log(T/(Kn_{a,t-1}))$, and proved that MOSS 
attains the minimax lower bound established in \cite{auer1995gambling, auer2002nonstochastic}. 
Unfortunately, the pioneering paper \cite{lai1987adaptive} was not cited early on 
in this proliferate literature.

\subsection{Recent developments}
\cite{auer2002finite} initiated the non-asymptotic analysis of UCB indices in the setting of nonparametric reward distributions. For multi-armed bandits with rewards bounded in $[0,1]$, they consider the following index policy, 
\begin{align}\label{Auer-index}
\phi_t = \argmax_{1 \le a \le K} \bigg\{\muhat_{a,t-1}+\sqrt{\alpha\log(t)/n_{a,t-1}}\bigg\}, \ \ t>K,
\end{align}
where $\muhat_{a,t-1}$ denotes the average reward of arm $a$ at time $t-1$ and $\alpha$ is some constant. They established the following elegant regret bound for the index with $\alpha=2$.
\bel{Auer-bd}
      R_T \le 8\sum_{a: \mu_a< \mu^*} \f{\log T}{\mu^* - \mu_a}+\Big(1+\f{\pi^2}{3}\Big)\sum_{a:\mu_a < \mu^* }(\mu^* - \mu_a).
\eel
This bound is logarithmic in $T$ and only includes constant factors in its second term. However, the leading constant factor 8 of $\log T$ in \eqref{Auer-bd} is bigger than the optimal constant factor $1/2$ 
for this index because the maximum variance is $1/4$ for rewards in $[0,1]$.  
\cite{bubeck2010bandits} established a regret bound 
with a leading constant factor $\alpha$ for any $\alpha>1/2$. 
\cite{audibert2009exploration} proposed UCB indices based on empirical variances 
to achieve leading constants that depend on the variances of arms. 

Although the UCB indices mentioned above enjoy non-asymptotic regret guarantees for bounded rewards in $[0,1]$, they do not satisfy the asymptotic lower bound based on minimum KL divergence 
\citep{burnetas1996optimal}. \cite{honda2010asymptotically,honda2015non} developed asymptotically optimal algorithms based on minimum empirical divergence (MED) for bounded rewards in $[0,1]$, 
but their algorithm is not index based. 
\cite{cappe2013kl} studied the use of UCB-type polices to achieve the minimum KL lower bound 
for rewards in $[0,1]$. Moreover, \cite{bubeck2013bandits} developed robust-UCB methods 
for bandits with heavy-tailed rewards.  

In the parametric case, building on the previous works \citep{garivier2011kl, maillard2011finite}, \cite{cappe2013kl} developed non-asymptotic regret bounds of kl-UCB, as defined in \eqref{cape-kl-ucb}, for bandits with univariate exponential family rewards.
\cite{honda2019note} provided 
asymptotic guarantee of kl-UCB+ for Bernoulli rewards. 
\cite{kaufmann2018bayesian} established non-asymptotic regret bounds for variants of Lai's UCB index and also generalized the lower bound in \cite{lai1987adaptive} for Bayes risk with product priors. 

More recently, an active line of research is the development of bi-optimal UCB indices that are both minimax and asymptotically optimal for multi-armed bandits. 
\cite{menard2017minimax} showed that a variant of kl-UCB called kl-UCB++ is bi-optimal for 
reward distributions in univariate exponential families. 
\cite{lattimore2018refining} introduced Ada-UCB for Gaussian rewards to achieve a strong non-asymptotic regret bound. \cite{garivier2022kl} developed a bi-optimal UCB index combining MOSS \citep{audibert2009minimax} and KL-UCB \citep{cappe2013kl} for rewards bounded in $[0,1]$. 

Apart from UCB-type policies, Thompson sampling \citep{thompson1933likelihood} has emerged as another prominent algorithm due to its strong empirical performance \citep{chapelle2011empirical}. Non-asymptotic analysis of Thompson sampling was carried out in \cite{agrawal2012analysis, agrawal2017near}. Additionally, the asymptotic optimality of Thompson sampling 
was established in \cite{kaufmann2012thompson} and \cite{korda2013thompson} 
for reward distributions in univariate exponential families.   Other asymptotic optimal policies include BayesUCB \citep{kaufmann2012bayesian, kaufmann2018bayesian} in the univariate exponential family case, ISM \citep{cowan2017normal}) for Gaussian rewards with unknown means and variances, and algorithms based on sub-sampling \citep{baransi2014sub,chan2020multi}. Readers are referred to \cite{bubeck2012regret,lattimore2020bandit} for detailed references.

\subsection{Our contributions}
In this paper, we establish non-asymptotic regret bounds for two UCB indices with a fixed horizon 
for Gaussian rewards. 
First, we consider the following UCB index with a constant exploration function, 
\begin{align*}
	\phi_t = \argmax_{1 \le a \le K}\bigg\{\muhat_{a,t-1}+\frac{\sigma b_{T'}}{\sqrt{n_{a,t-1}}}\bigg\},	\ \ t>K,
\end{align*}
where $T' = T-K$ and $b_{T'}$ is a constant depending on $T'$. This can be viewed as the choice of replacing $g(T/n_{a,t-1})$ by $b_{T'}^2/2$ in \eqref{lai-index}. 
For $T' \ge 100$ and $b_{T'} = \sqrt{2\log T'}$, our regret bound can be stated as
\begin{align*}
	R_T	\le \sum_{a:\mu_a<\mu^*}\frac{\sigma^2(2\log T'+4)}{\mu^* - \mu_a} + \sum_{a:\mu_a <\mu^*}(\mu^* - \mu_a).
\end{align*}
Notice that the regret bound has a leading constant matching the Lai-Robbins lower bound. 
Additionally, our theory shows that a suitable choice of $b_{T'}$ will lead to a regret bound with negative lower order terms. Similar regret bounds were obtained by \cite{honda2015non, garivier2022kl} for rewards bounded in $[0,1]$. 

Our second contribution is a non-asymptotic regret bound for a specific instance of Lai's UCB index, 
which can be also viewed as the kl-UCB+ \citep{garivier2011kl} for a fixed horizon. 
We do not require an additional $\log\log(T)$ term in the exploration function 
as in the $\text{kl-UCB-H}^+$ in \cite{kaufmann2018bayesian}. 
\cite{honda2019note} proved the asymptotic optimality of kl-UCB+ in the Bernoulli case. 
In comparison, our regret bounds are fully non-asymptotic with sharp constant factor in the leading term and bounded second order term. 

We took a different analytical approach compared with existing ones. 
 A main issue in our analysis is to bound the probability 
 for a random walk to cross a square-root boundary. 
We treat this boundary crossing probability as the Type I error of a repeated significance test 
\citep{woodroofe1979repeated, siegmund1985sequential, siegmund1986boundary} 
and apply a non-asymptotic version of the nonlinear renewal theory 
\citep{lai1977nonlinear, lai1979nonlinear, woodroofe1982nonlinear, zhang1988nonlinear} 
instead of directly using a result in \cite{lerche2013boundary} as in \cite{lattimore2018refining}. 
Interestingly, in addition to multi-armed bandits,  
the square--root boundary is connected to the repeated significance test in clinical trials \citep{armitage1960sequential} and optimal stopping 
for random walks \citep{chow1965optimal,chow1971great} and Brownian motion \citep{shepp1969explicit}. 
 

\subsection{Organization}
The rest of this paper is organized as follows. 
Section \ref{Gaussian-sec} presents the non-asymptotic regret bounds of UCB indices. Section \ref{sec:proof-regret} presents the proofs of our regret bounds. 
Section \ref{sec:tech-lemma} provides some technical lemmas and their proofs.

\section{Main results}\label{Gaussian-sec} 
In this section, we present sharp regret bounds of a UCB index with a constant level of exploration under the fixed horizon and a similar non-asymptotic regret bound for Lai's UCB index. 
\subsection{Problem setting}\label{sec-prob-setting}
We focus on a $K$-armed bandit problem with a fixed time horizon $T$, $2 \le K \le T$, 
and assume that 
the rewards sampled from arm $a$ are independent and identically distributed 
Gaussian random variables with mean $\mu_a$ and a variance no greater than $\sigma^2$. 
Let $y_{t}$ denote the reward received at each time $t$ and $\scrF_t = \sigma(y_1,\ldots, y_{t})$. 
An allocation rule $\{\phi_t\}_{t=1}^{T}$, $\phi_t \in \{1,\ldots,K\}$, is adaptive if $\phi_t$ is $\scrF_{t-1}$ measurable for each $t$.  We assume 
\begin{align*}
y_{t}\big|\scrF_{t-1} \sim y_{t}\big|\phi_{t},\quad 
\E\big[y_{t}\big|\phi_{t}=a\big] = \mu_a. 
\end{align*}
We denote the maximal mean among arms by $\mu^* = \max_{1 \le a \le K} \mu_a$ 
and the optimal arm by $a^* = \argmax_{1 \le a \le K} \mu_a$ with an arbitrary tie-breaking rule. 
All allocation rules considered in this paper are initialized by $\{\phi_t, 1\le t\le K\}=\{1, \ldots, K\}$. Let $T'=T-K$ and $\Delta_a = \mu^* - \mu_a$. The sample size of arm $a$ at time $t$ is denoted as $n_{a,t} = \sum_{j=1}^t \ind\{\phi_{j} = a\}$. The cumulative regret after the initialization is defined as follows,
\begin{align}\label{regret-def}
R_{T'} = T'\mu^* - \E\Big[\sum_{t=K+1}^T y_t\Big]= \sum_{a= 1}^K \Delta_a \E[n_{a,T}-1],
\end{align}
where the last equality follows from conditioning. 

Throughout the paper, we use $\varphi(x)$ and $\Phi(x)$ 
to denote the standard Gaussian density and cumulative distribution functions respectively, and $\{W(t),t\ge 0\}$ to denote a standard Brownian motion. In addition, $x_+ =\max(x,0)$ for real $x$, and $a\wedge b = \min\{a,b\}$ for reals $a$ and $b$.


\subsection{Regret bounds for UCB with a constant level of exploration}
Let $b_{T'}$ be a constant level of exploration depending on $T'$ and define the following UCB index 
\begin{align}\label{UCB-rule-T}
  \phi_t = \argmax_{1 \le a \le K}\left\{ \muhat_{a,t-1} +  \frac{\sigma b_{T'}}{ \sqrt{n_{a,t-1}}} \right\},\quad t > K, 
\end{align}
with an initialization $\{\phi_t, 1\le t\le K\}=\{1, \ldots, K\}$,
where $\sigma$ is a prespecified noise level, $n_{a,t-1}$ is the sample size and $\muhat_{a,t-1}$ is the average rewards of arm $a$ at time $t-1$ after $y_{t-1}$ is sampled.
An arbitrary tie-breaking rule is applied to address multiple maxima in \eqref{UCB-rule-T}.
Define 
\begin{align}
\Phi^*(x,T') &= \P\Big\{-\max_{1\le m\le T'}W(m)/\sqrt{m} \le -x\Big\},\nonumber\\
\Phi_2(x) & = \int (z+x)_+^2\varphi(z)dz, 
\label{th-Gaussian-LaiRobbins-T-1} 
  \\  \nonumber
    \eta(b_{T'}) &= 4T'\Phi_2(-b_{T'}) + 3(b_{T'}^2+1)\Phi^*(-b_{T'},T'). 
    \end{align}
We have the following regret upper bound for the allocation rule \eqref{UCB-rule-T}.

\begin{theorem}\label{th-Gaussian-LaiRobbins-T} 
 Suppose the rewards from arm $a$ follow a Gaussian distribution with mean $\mu_a$ and no greater variance than $\sigma^2$ for all $a=1,\ldots,K$. Then, the regret of the UCB rule \eqref{UCB-rule-T} is bounded by
\bel{th-Gaussian-LaiRobbins-T-2} 
R_{T'}
\le  \sum_{a:\mu_a<\mu^*}\frac{\sigma^2(b_{T'}^2+1+\eta(b_{T'}))}{\mu^* - \mu_a},
\eel
where $R_{T'}$ is defined in \eqref{regret-def} and $\eta(b_{T'})$ is defined in \eqref{th-Gaussian-LaiRobbins-T-1}. 
\end{theorem}
\begin{remark}
  In the numerator of the right-hand side of \eqref{th-Gaussian-LaiRobbins-T-2}, the term $b_{T'}^2$ represents the leading term, and $\eta(b_{T'})$ is $o(1)$ as $T \rightarrow \infty$ by choosing $b_{T'}$ properly. The component $\Phi^*(-b_{T'},T')$ within $\eta(b_{T'})$ in \eqref{th-Gaussian-LaiRobbins-T-1} corresponds to the boundary crossing probability of Brownian motion, which can be interpreted as the size of a repeated significance test. For detailed studies, see \cite{woodroofe1979repeated} and \cite{siegmund1985sequential}. A non-asymptotic upper bound for this probability is provided in Lemma \ref{lm-Gaussian-boundary} 
in Section \ref{sec:tech-lemma}.
\end{remark}
Theorem \ref{th-Gaussian-LaiRobbins-T}, combined with numerical evaluations, leads to the following corollary. 
\begin{corollary}\label{cor-Gaussian-LaiRobbins-T-1}
  Setting $b_{T'}=\sqrt{2\log T'}$ in \eqref{UCB-rule-T}, we find that
    \begin{align*} 
    R_{T'}
    \le \sum_{a:\mu_a<\mu^*}\frac{\sigma^2(2\log T'+c_1(T'))}{\mu^* - \mu_a},
    \end{align*}
    where $c_1(T') =o(1)$ as $T' \rightarrow \infty$ and
    $c_1(T') \le  10.1, 7, 5.5, 4, 3\ldots$ for $T' \ge 2, 20, 40, 100, 200\ldots$.
    \end{corollary}
\begin{remark}
Corollary \ref{cor-Gaussian-LaiRobbins-T-1} establishes a sharp non-asymptotic regret bound with optimal leading constant, which implies that the UCB rule achieves the 
information lower bound of \cite{lai1985asymptotically}. In fact, for the choice $b_{T'}$ in the above corollary this optimality is uniform in the sense of
\begin{align*}
     \limsup_{T'\to\infty} \sup_{\mu\in \R^K}
     \frac{R_{T'}}{(\log T')\sum_{\mu_a<\mu^*}2\sigma^2/(\mu^* - \mu_a)} \le 1,
\end{align*}
where $\mu = (\mu_1,\ldots, \mu_K)$.
\end{remark}
    
According to Lemma \ref{lm-Gaussian-boundary} in Section \ref{sec:tech-lemma},
$\Phi^*(-b_{T'},T')\lesssim (\log T')^{3/2}/T'$ in \eqref{th-Gaussian-LaiRobbins-T-1} for $b_{T'} = \sqrt{2\log T'}$, 
so that the second term in \eqref{th-Gaussian-LaiRobbins-T-1} is negligible as $T' \rightarrow \infty$. Therefore,  Theorem \ref{th-Gaussian-LaiRobbins-T} suggests the use of 
the exploration level
    \bel{th-Gaussian-LaiRobbins-T-4}
    b_{T'} = \argmin_{z>0}\big\{z^2 + 4T'\Phi_2(-z)\big\}. 
    \eel
    \begin{corollary}\label{cor-Gaussian-LaiRobbins-T-2}
      The UCB rule \eqref{UCB-rule-T} with the exploration level $b_{T'}$ in \eqref{th-Gaussian-LaiRobbins-T-4} enjoys the following regret bound, 
    \begin{align*}
    R_{T'}
    \le \sum_{a:\mu_a<\mu^*}\frac{\sigma^2(2\log T' - 3\log\log(T') -\log \pi +1  +\eps_{T'})}{\mu^* - \mu_a} 
    \end{align*}
    for some $\eps_{T'} = o(1)$ depending on $T'$ only.
    \end{corollary}


\subsection{Regret bound for Lai's UCB}
In this section, we consider Lai's UCB index in \eqref{lai-index} with $g(x)=1\vee \log x$ 
for Gaussian rewards.  
Let $T'=T-K$. With an initialization $\{\phi_t, 1\le t\le K\}=\{1, \ldots, K\}$, Lai's UCB rule can be written as
\begin{align}\label{UCB-rule-Lai}
\phi_t = \argmax_{1 \le a \le K}\left\{ \muhat_{a,t-1} 
+ \sigma \sqrt{\f{2\log_+(T'/n_{a,t-1})}{n_{a,t-1}}} \right\},\quad t > K, 
\end{align}
where $\log_+(x)=1\vee\log x$, $n_{a,t-1}$ and $\muhat_{a,t-1}$ are defined as in \eqref{UCB-rule-T}, and $\sigma$ is a prespecified noise level. Again any tie-breaking rule can be applied in \eqref{UCB-rule-Lai}.  

\begin{theorem}\label{th-Gaussian-Lai-T} 
  Suppose the rewards from arm $a$ follow a Gaussian distribution with mean $\mu_a$ 
  and no greater variance than $\sigma^2$ for all $a=1,\ldots,K$. 
  Let $\gamma_a = (\mu^* - \mu_a)/\sigma$. 
  Then, the UCB index policy in \eqref{UCB-rule-Lai} satisfies
\bel{th-Gaussian-Lai-T-1} 
\quad &&	R_{T'} \le  \sum_{a:\mu_a< \mu^*}
	\frac{\sigma^2\big(2L(T'\gamma_a^2)+1+\eps_{a,T'}\big)}{\mu^* - \mu_a}
\eel
where $R_{T'}$ is defined in \eqref{regret-def}, 
$L(x) = \log_+(x/\log_+(x/\log_+(x)))$, and $\eps_{a,T'}$ is uniformly bounded with 
$\eps_{a,T'}\le 14.8$ and  $\eps_{a,T'}\to 0$ as $T'\gamma_a^2\to\infty$. 
\end{theorem}

\begin{remark}
  \cite{kaufmann2018bayesian} established non-asymptotic regret bounds for the UCB index rule in \eqref{UCB-rule-Lai} with an exploration function $\log(T/n_{a,t-1})+7\log\log(T/n_{a,t-1})$ for rewards of univariate exponential families. However, their regret bound does not have a sharp leading constant and has $O(\sqrt{\log T})$ as lower order terms. 
\end{remark}

\section{Proofs of regret bounds}\label{sec:proof-regret}
We provide here the proofs of the regret upper bounds in 
the main theorems and corollaries presented in Section \ref{Gaussian-sec}. 
The following notation will be used throughout this section. We define $y_{a,n} = \mu_a + \veps_{a,n}$ as the $n$-th sample from arm $a \in \{1, \cdots, K\}$, $\ybar_{a,n} = n^{-1}\sum_{i=1}^n y_{a,i}$ as the sample average and $\vepsbar_{a,n} =\ybar_{a,n} - \mu_a$. For suboptimal arms $a$, we write $\Delta_a = \mu^*-\mu_a$ and $\gamma_a=(\mu^*-\mu_a)/\sigma$.

\subsection{Proof of Theorem \ref{th-Gaussian-LaiRobbins-T}}
According to the above notation, the UCB index can be expressed as 
\begin{equation*}
  \phi_t = \argmax_{1 \le a \le K}\{\ybar_{a,n_{a,t-1}}+\sigma b_{T'}/\sqrt{n_{a,t-1}}\},\ \ t>K.
\end{equation*}  
    For the optimal arm $a^*$, let $X^*=\min_{1\le m \le T'}\big(\vepsbar_{a^*,m}/\sigma + b_{T'}/m^{1/2}\big)$ 
    and $P(x)$ be its distribution function $\P\{X^*\le x \}$. We have 
    \bes
    \E[n_{a,T}-1|X^*=x] 
           &\le& \sum_{n=1}^{T'}\P\Big\{\ybar_{a,n} + \sigma b_{T'}/n^{1/2} \ge \mu^* + \sigma x\Big\}
         \cr &\le& \sum_{n=1}^{T'}\P\Big\{(Z + b_{T'})/n^{1/2} \ge \gamma_a + x \Big\}
    \cr &\le & \min\{T', g_{T'}(\gamma_a+x)\}, 
    \ees
    where $Z\sim N(0,1)$ and $g_{T'}(x) = (b_{T'}^2+1)/x_+^2 = \E[(Z+b_{T'})^2/x_+^2$.
    Because $g_{T'}(x)$ is a bounded nonnegative non-increasing differentiable function of $x$,  
    \bel{pf-th-Gaussian-LaiRobbins-T-1}
    && \E[n_{a,T}-1] 
    \cr &\le& \int \min\{T', g_{T'}(\gamma_a+x)\} P(dx) 
    \cr &\le& g_{T'}(\gamma_a)(1-P(0)) + T'P(-\gamma_a/2) + \int_{-\gamma_a/2}^0 g_{T'}(\gamma_a+x)P(dx)
    \cr &=& g_{T'}(\gamma_a)  + T'P(-\gamma_a/2) - \int_{-\gamma_a/2}^0 P(x)g_{T'}'(\gamma_a+x)dx. 
    \eel
    By Lemma \ref{lm-Gaussian-crude-T}, $P(-\gamma ) \le \Phi_2(-b_{T'})/\gamma^2$ for all $\gamma \ge 0$, so that 
    \bel{pf-th-Gaussian-LaiRobbins-T-2}
    && T'P(-\gamma_a/2) - \int_{-\gamma_a/2}^0 P(x)g_{T'}'(\gamma_a+x)dx 
    \cr &\le& \frac{4T'\Phi_2(-b_{T'})}{\gamma_a^2} + \int_{-\gamma_a/2}^0 \frac{2P(0)(b_{T'}^2+1)}{(\gamma_a+x)^3}dx
    \cr &=& \frac{\eta(b_{T'})}{\gamma_a^2}
    \eel
    in view of \eqref{th-Gaussian-LaiRobbins-T-1} and the fact that $P(0)=\Phi^*(-b_{T'},T')$. 
    The conclusion follows from \eqref{pf-th-Gaussian-LaiRobbins-T-1}.

\subsection{Proof of Corollary \ref{cor-Gaussian-LaiRobbins-T-1}}


Inserting $b_{T'} = \sqrt{2\log T'}$ into the upper bound of Lemma \ref{lm-Gaussian-boundary} yields $\eta(b_{T'}) = o(1)$ for $T' \rightarrow \infty$. According to the proof of Lemma \ref{lm-Gaussian-boundary}, $\Phi^*(-b_{T'}, T')$ can be bounded by the integrals in \eqref{pf-lm-repeated}. Numerical evaluations of \eqref{pf-lm-repeated} and $\Phi_2(-b_{T'})$ for various values of $T'$ lead to our conclusion.


\subsection{Proof of Corollary \ref{cor-Gaussian-LaiRobbins-T-2}}
As $ b_{T'}$ is defined implicitly in \eqref{th-Gaussian-LaiRobbins-T-4}, we first derive an 
expansion of it. 
Let 
\bes
f_{T'}(z) = z^2 + 4T'\Phi_2(-z) = z^2 + 4T' \int (x-z)_+^2\varphi(x)dx,
\ees 
so that $b_{T'} = \argmin_{z>0} f_{T'}(z)$. 
As $f_{T'}''(z) \ge 2$, $f_{T'}(z)$ is strictly convex in $z$ and the solution $b_{T'}$ is uniquely the solution of $f'_{T'}(z)=0$ or equivalently the solution of 
\bes
z &=& 4T'\int_{z}^\infty (x-z)\varphi(x)dx
= \frac{4T' \varphi(z)}{z^2} \int_0^\infty x e^{-x^2/(2z^2) - x} dx. 
\ees
As $\int_0^\infty x e^{ - x} dx=1$, 
it follows that $b_{T'} >0$ for all $T'$, and $b_{T'} \to\infty$ and 
\bes
b_{T'}^2/2 &=& \log\bigg(\frac{4T'}{\sqrt{2\pi}b_{T'}^3} \int_0^\infty x e^{-x^2/(2b_{T'}^2) - x} dx\bigg)
\cr &=& \log T' - \log\Big((b_{T'}^2/2)^{3/2}\Big) + \log\bigg(\frac{4}{\sqrt{2\pi}2^{3/2}}\bigg) + o(1)
\cr &=& \log T' - \frac{3}{2}\log\log T' - \frac{1}{2}\log\pi + o(1) 
\ees
as $T'\to\infty$. 
 The conclusion is deduced from Theorem \ref{th-Gaussian-LaiRobbins-T}.

\subsection{Proof of Theorem \ref{th-Gaussian-Lai-T}}
Let $Y^*=\min_{1\le m \le T'}\big(\vepsbar_{a^*,m}/\sigma + \sqrt{(2/m)\log_+(T'/m)}\big)$
and $Q(y)$ be the distribution function of $Y^*$, $Q(y)=\P\{Y^*\le y \}$. We have
\begin{align}
	\E\big[n_{a,T} -1\big] & \le  \sum_{n = 1}^{T'}
	\P\bigg\{\ybar_{a,n} + \sqrt{\f{2 \sigma^2\log_+(T'/n)}{n}} \ge \mu^*+ \sigma Y^*	\bigg\}
	\nonumber\\
	& = \int\Bigg[\sum_{n = 1}^{T'}
	\P\bigg\{\vepsbar_{a,n}/\sigma + \sqrt{\f{2 \log_+(T'/n)}{n}} \ge \gamma_a +y	\bigg\}\Bigg] Q(dy)
\label{eqn:pf-lai-ucb-0}
\end{align}
as in the proof of Theorem \ref{th-Gaussian-LaiRobbins-T}, where $\gamma_a = (\mu^*-\mu_a)/\sigma$. 

Let $\gamma>0$ satisfying $T'\gamma^2 > e$, $Z\sim N(0,1)$, $\xbar(\gamma)$ be the solution of 
$\gamma^2 \xbar(\gamma) = \log_+(T'/\xbar(\gamma))$, 
and $\bbar(\gamma) = \gamma\sqrt{2\xbar(\gamma)} = \sqrt{2\log_+(T'/\xbar(\gamma))}$. 
Because $\sqrt{n}\vepsbar_{a,n}/\sigma\sim -Z$ and 
$\sqrt{2\log(T'/n)}\le \bbar(\gamma)$ for $n\ge \xbar(\gamma)$, 
\bes
&& \sum_{n = 1}^{T'}
	\P\bigg\{\vepsbar_{a,n}/\sigma + \sqrt{\f{2 \log_+(T'/n)}{n}} \ge \gamma	\bigg\}
\cr &\le& \lfloor \xbar(\gamma)\rfloor  + \sum_{n = \lfloor \xbar(\gamma)\rfloor+1}^\infty
	\P\big\{ \bbar(\gamma) - Z \ge \gamma \sqrt{n}	\big\}
\cr &\le& \xbar(\gamma)
+ \E\big[\big((\bbar(\gamma) - Z)^2/\gamma^2 - \xbar(\gamma) \big)
I\{\bbar(\gamma) - Z \ge\gamma \sqrt{\xbar(\gamma)}\}\big]
\cr &=& \gamma^{-2}\big(\bbar^2(\gamma)+1 + \cbar(\bbar(\gamma))\big), 
\ees
where $\cbar(b) = \E\big[\big(b^2/2 - (Z - b)^2\big)I\{ Z > b'\}\big]
=(2b-b')\varphi(b') -  (1+b^2/2)\Phi(-b')$ 
with $b'=b(1-1/\sqrt{2})$. 
As $\bbar(\gamma)\ge \sqrt{2}$, 
$\cbar(\bbar(\gamma))
\le \max_{b\ge \sqrt{2}}\cbar(b)< 0.3487$. 

Define $\gbar(\gamma) = \gamma^{-2}\big(\bbar^2(\gamma)+1+\cbar(\bbar(\gamma))\big)$. 
By \eqref{eqn:pf-lai-ucb-0} and the above inequality, 
\bes
\E\big[n_{a,T} -1\big] 
&\le& \int \min\big(T', \gbar((\gamma_a+y)\wedge \gamma_a)\big) Q(dy) 
\cr &\le& \int_{-\eta\gamma_a}^\infty \gbar_a((\gamma_a+y)\wedge \gamma_a) Q(dy) + T'Q(-\eta\gamma_a)
\ees
for any $\eta\in (0,1)$, where $\gbar_a(\gamma) =\gamma^{-2}\big(\bbar^2(\gamma)+1+\cbar_a\big)
= 2\xbar(\gamma)+\gamma^{-2}(1+\cbar_a)$ with  
$\cbar_a = \max_{-\eta \le y\le 0}\cbar(\bbar(\gamma_a(1+y)))$. 
As $\xbar'(\gamma) = - 2\gamma \xbar^2(\gamma)/(\gamma^2\xbar(\gamma) +1)$, 
\bes
&& \int_{-\eta\gamma_a}^\infty \gbar((\gamma_a+y)\wedge \gamma_a) Q(dy) 
- \gbar(\gamma_a)
\cr &\le& \int_{-\eta\gamma_a}^0
\bigg( \frac{4\gamma \xbar^2(\gamma)}{\gamma^2\xbar(\gamma) +1)}\bigg|_{\gamma=\gamma_a+y}
+ \frac{2(1+\cbar_a)}{(\gamma_a+y)^3}\bigg)Q(y)dy
\cr &=& \int_0^\eta
\bigg(\gamma^2 \xbar(\gamma)+ \frac{1}{\gamma^2 \xbar(\gamma) +1} 
+ \frac{\cbar_a-1}{2}\bigg)\bigg|_{\gamma=\gamma_a(1-y)}\frac{4Q(-\gamma_ay)}{\gamma_a^2(1-y)^3}dy. 
\ees

By Lemma \ref{lm:mixed-boundary-gauss}, 
$Q(-\gamma)= c_0(T'\gamma^2)/(T'\gamma^2)\le 1.7068/(T'\gamma^2)$ for all $\gamma>0$. 
Let $L_a(y)=L(\kappa_a(1-y)^2)$ 
with $\kappa_a = T'\gamma_a^2$ and 
$L(x) = \log_+(x/\log_+(x/\log_+(x)))$. 
As $\gamma^2\xbar(\gamma) = \log_+(T'\gamma^2/(\gamma^2\xbar(\gamma)))\le L(T'\gamma^2)$, 
the above integral is bounded by 
\bes
J(\kappa_a,\eta) = 4\int_0^\eta
\frac{L_a(y)+1/(L_a(y)+1)+(\cbar_a-1)/2}{(1-y)^3}
\bigg(1\wedge \frac{1.7068}{\kappa_ay^2}\bigg)dy. 
\ees

Moreover, because $\gamma_a^2\gbar_a(\gamma_a)\le 2L(\kappa_a)+1+\cbar_a$,
we find that  
\bes
&& \gamma_a^2 \E[n_{a,T}-1] - 2L(\kappa_a) - 1 
\le \cbar_a + J(\kappa_a,\eta) + c_0(\eta^2 \kappa_a)/\eta^2
\ees
for any choice of $\eta\in (0,1)$. 
For $\eta=0.573$ and $\kappa_a \ge 20.47$, $c_0(\eta^2\kappa_a) \le 1.7068$ by Lemma \ref{lm:mixed-boundary-gauss} and
the right-hand side above is no greater than 14.8, so that 
\bel{eqn:pf-lai-ucb-1}
&& \E[n_{a,T}-1] 
\le \gamma_a^{-2}\big(2L(\kappa_a) + 1 + \eps_{a,T'}\big)
\eel
with $\eps_{a,T'} < 14.8$. For $\kappa_a \le 20.47$, 
\bes
\frac{2L(\kappa_a)+1+14.8}{T'\gamma_a^2}
\ge \frac{2L(20.47)+15.8}{20.47} > 1
\ees
so that \eqref{eqn:pf-lai-ucb-1} holds with $\eps_{a,T'} \le 14.
8$ anyways. 
For fixed $\eta\in (0,1)$, $J(\kappa_a,\eta)\lesssim (\log\kappa_a)^2/\kappa_a=o(1)$ 
and Lemma \ref{lm:mixed-boundary-gauss} provides $c_0(\kappa_a\eta^2)=o(1)$ as $\kappa_a\to\infty$. 
Because $\cbar(b)\to 0$ when $b\to\infty$ and 
$\bbar(\gamma) = \sqrt{2\gamma^2\xbar(\gamma)}\to\infty$ when  $T'\gamma^2\to \infty$, 
$\cbar_a 
\to 0$ when $\kappa_a\to\infty$.
Thus, $\eps_{a,T'}\to 0$ in \eqref{eqn:pf-lai-ucb-1} when $\kappa_a\to\infty$. 
The conclusion follows directly from \eqref{eqn:pf-lai-ucb-1} 
as $(\mu^*-\mu_a)/\gamma_a^2 = \sigma^2/(\mu^*-\mu_a)$.

\section{Technical lemmas}\label{sec:tech-lemma}
In this section, we provide some inequalities for boundary crossing probabilities. 
Among them, Lemmas \ref{lm-Gaussian-crude-T} and \ref{lm-Gaussian-boundary} are 
used in the proof of Theorems \ref{th-Gaussian-LaiRobbins-T}, 
and Lemma \ref{lm:mixed-boundary-gauss} is used in the proof of  
Theorem \ref{th-Gaussian-Lai-T}. 

 Our first lemma deals with the square root boundary crossing for a Brownian motion with drift $-\gamma$.
 
    \begin{lemma}\label{lm-Gaussian-crude-T} 
    Let $W(m)\sim N(0,m)$. Then, for all $b >0$ and $\gamma>0$  
    \bel{lm-Gaussian-crude-1}
    && \P\bigg\{\sup_{m\ge 1} \frac{W(m)}{\sqrt{m}b + m\gamma} \ge 1\bigg\} 
    \le \frac{\Phi_2(-b)}{\gamma^2}, 
    \eel
    where $\Phi_2(x)$ is defined as in \eqref{th-Gaussian-LaiRobbins-T-1}.
    \end{lemma}
    
    \begin{proof} As $W(m)/m^{1/2}\sim N(0,1)$, the union bound gives 
    \bes
    \P\bigg\{\sup_{m\ge 1} \frac{W(m)}{\sqrt{m}b + m\gamma} \ge 1\bigg\} 
    \le\int_0^\infty \int_{b + \sqrt{x}\gamma}^\infty \varphi(z)dzdx 
    = \int_0^\infty \frac{(z-b)_+^2}{\gamma^2}\varphi(z)dz. 
    \ees
    The right-hand side equals $\Phi_2(-b)/\gamma^2$. 
    \end{proof}
    

We need the following inequalty for an expected stopping rule 
in the proof of Lemma \ref{lm-Gaussian-boundary}.

\begin{lemma}\label{lm-Gaussian-mean-stopping}
Let $X(t) = W(t) + \theta t$ be a Brownian motion with drift $\theta$ under $\P_\theta$. Define 
\bes
	\tau_b = \inf\big\{t\ge 1: |X(t)| \ge bt^{1/2}\big\}, 
\ees
$\theta>0$, $g_b(\theta) = \E_\theta[\theta - X(1) | \tau_b >1]$ 
and $t_\theta = \big(b+\sqrt{b^2+4\theta g_b(\theta)}\big)/(2\theta)$. Then, 
\bel{lm-Gaussian-mean-stopping-1}
	\E_\theta\big[\sqrt{\tau_b}\big|\tau_b >1 \big]
	\le t_\theta \le b/\theta + \min\big\{g_b(\theta)/b, \sqrt{g_b(\theta)/\theta}\big\}. 
\eel
\end{lemma}

\begin{proof}
By definition $t_\theta$ is the solution of $\theta t_\theta^2 = bt_\theta + g_b(\theta)$. 
As $g_b(\theta) > (\theta - b)_+$, $t_\theta>1$. 
As $bt^{1/2} = bt_\theta(t/t_\theta^2)^{1/2}\le bt_\theta(1+ t/t_\theta^2)/2$, 
\bes
	\tau_b\le \tau_b' = \inf\Big\{t\ge 1: X(t) \ge bt_\theta(1+ t/t_\theta^2)/2\Big\}. 
\ees
It follows from Wald's identity that 
\bes
	\E_\theta\big[\theta(\tau_b'-1)\big|\tau_b >1 \big]
	&=& \E_\theta\big[ X(\tau_b') - X(1) \big|\tau_b >1 \big] 
	\cr &=& \E_\theta\big[bt_\theta(1+ \tau_b'/t_\theta^2)/2\big|\tau_b >1\big] + g_b(\theta)-\theta, 
\ees
which is equivalent to 
$2\theta \E_\theta\big[\tau_b'\big|\tau_b >1 \big]
= (b/t_\theta) \E_\theta\big[\tau_b' \big|\tau_b >1\big] + bt_\theta+ 2g_b(\theta)$. 
Because $b/t_\theta <\theta$, the unique solution of the above equation is 
$\E_\theta\big[\tau_b'\big|\tau_b >1 \big] = t_\theta^2$. It follows that 
$\E_\theta\big[\sqrt{\tau_b}\big|\tau_b >1 \big]
\le \E_\theta\big[\sqrt{\tau_b'}\big|\tau_b >1 \big] \le t_\theta$. 
\end{proof} 

As Lemma \ref{lm-Gaussian-crude-T}, the following lemma deals with the driftless case.
\begin{lemma}\label{lm-Gaussian-boundary} Let $W(t)$ be a standard Brownian motion under $\P$. 
Let $\varphi(x) = e^{-x^2/2}/\sqrt{2\pi}$. 
For all real numbers $b>0$ and $0<m_0 < m$, 
\bes
	\P\big\{ \max_{m_0 \le t \le m} |W(t)|/t^{1/2} > b \big\} 
	\le \varphi(b)\big\{2b\log(e(m/m_0)^{1/2})+\sqrt{2/\pi} + 4/b)\big\}. 
\ees
\end{lemma} 

\begin{proof} 
Assume without loss of generality $m_0=1$ as $W(m_0 t)/m_0^{1/2}$ is a Brownian motion. 
Let $X(t) = W(t) + \theta t$ under $\P_\theta$ and 
\bes
	\tau = \tau_b = \inf\big\{t\ge m_0: |X(t)| \ge bt^{1/2}\big\}. 
\ees
Let $\scrF_t$ be the sigma-field generated by $\{X(s), s\le t\}$. 
The likelihood ratio $d\P_{\theta}/d\P_0$ in $\scrF_{\tau}$ is 
$\exp[ \theta X(\tau) - \tau\theta^2/2]$ and 
$\int \exp[ \theta X(\tau) - \tau\theta^2/2]d\theta = \sqrt{2\pi/\tau}
\exp[X^2(\tau)/(2\tau)]$. Thus, Wald's likelihood ratio argument provides 
\begin{align*}
	&\P_0\big\{1 < \tau \le m\big\}\\ 
	&= \E_0\bigg[ e^{-b^2/2}\sqrt{\tau/(2\pi)} \int \exp\big[\theta X(\tau) - \tau \theta^2/2\big] d\theta I\{1 < \tau \le m\}\bigg] \\
    &= \varphi(b) \int \E_\theta\Big[ \sqrt{\tau} I\big\{1 < \tau\le m\big\}\Big] d\theta. 
\end{align*}
With an application of Lemma \ref{lm-Gaussian-mean-stopping} and variable change $x = \theta/b$, we find that 
\begin{align}
	&\P_0\big\{1 < \tau \le m\big\}/(2\varphi(b))\nonumber\\
	&\le \int_0^\infty \min\bigg(m^{1/2}, \frac{b}{\theta}+\frac{g_b(\theta)}{b}\bigg)
	\P_\theta\big\{\tau_b>1\big\}  d\theta \nonumber 
	\\  
	&\le  b\int_0^\infty \big(m^{1/2}\wedge x^{-1}\big)\int_{|z+xb|\le b}\varphi(z)dz  dx
	+ \int_0^\infty \int_{|b-\theta|}^{b+\theta} \frac{z}{b}\varphi(z)dz d\theta. \label{pf-lm-repeated}
\end{align}
The first double integral on the right-hand side above is bounded by 
\bes
	\int_1^\infty x^{-1}\int_{|z+xb|\le b}\varphi(z)dz  dx
	= \int_0^\infty \log( 1+ z/b) \varphi(z)dz \le \frac{1}{b\sqrt{2\pi}}
\ees
and $\int_0^1 \min\big(m^{1/2}, x^{-1}\big)dx = 1+\log(m^{1/2})$, while the second is bounded by 
\bes
	\int_0^\infty \int_{|b-\theta|}^{b+\theta} \frac{z}{b}\varphi(z)dz d\theta
	= \int_0^\infty  \big(b+z-|b-z|\big)\frac{z}{b} \varphi(z)dz 
	\le \int_0^\infty \frac{2z^2}{b} \varphi(z)dz = \frac{1}{b}. 
\ees
Inserting the above bounds to \eqref{pf-lm-repeated}, we find that 
\bes
	\P_0\big\{1 < \tau \le m\big\} 
	\le \varphi(b)\big\{2b\log(em^{1/2})+\sqrt{2/\pi} + 2/b\big\}. 
\ees
The conclusion follows as $\P\{|W(1)|>b\}\le (2/b)\varphi(b)$. 
\end{proof}

Finally, our last lemma deals with the boundary of Lai's UCB, or equivalently 
the boundary for repeated significance test with slowly changing threshold level 
$\sqrt{2\log_+(n/t)}$. Recall that $\log_+(x)=\log(x\vee e)$. 

\begin{lemma}\label{lm:mixed-boundary-gauss}
	Let $W(t)$ be a standard Brownian motion. Then,  
	\bel{lm:mixed-boundary-gauss-1}
           n\gamma^2 \P\biggl\{ \sup_{0 < t\le n} \frac{|W(t)|}{\sqrt{2t\log_+(n/t)}+t\gamma}\ge 1 \biggr\} 
           = c_0(n\gamma^2)
	\eel
	for all positive $n$ and $\gamma$, with $\sup_{x>0}c_0(x)\le 1.7068$ 
	and $\lim_{x\to\infty}c_0(x)=0$. 
\end{lemma}

\begin{proof} 
We write the probability in \eqref{lm:mixed-boundary-gauss-1} as 
\bes
\P\biggl\{ \sup_{0 < s \le 1} \frac{|B(s)|}{\sqrt{2s\log_+(1/s)}+s n^{1/2}\gamma}\ge 1 \biggr\}
\ees
with $s=t/n$ and $B(s) = W(ns)/n^{1/2}$. 
As $B(s)$ is a standard Brownian motion, 
 the probability depends on $(n,\gamma)$ only through $n\gamma^2$. 
In what follows we assume without loss of generality $\gamma=1$ and $n\in (0,\infty)$. 

Let $b(t) = \sqrt{2\log_+(n/t)} + \sqrt{t}$ and $t_0=\argmin_{t>0} b(t)$. 
For $n \ge 2e$, $t_0$ is the unique solution of $\sqrt{t\log(n/t)} = \sqrt{2}$ 
in $(0,n/e]$. For $n < 2e$, $t_0=n/e$. 
The function $b(t)$ is decreasing in $(0,t_0]$ and increasing in $[t_0,\infty)$. 
Let 
\bes
X(t) = W(t)/(\sqrt{2t\log_+(n/t)}+t) = W(t)/(\sqrt{t}b(t)).
\ees 
The probability $\P\{|X(t)|\ge 1\} = 2\Phi(-b(t))$ is maximized at $t=t_0$. 
Because $b(t)\ge b(t_0) = 2/\sqrt{t_0}+\sqrt{t_0} \ge \sqrt{8}$, 
$b(t)e^{-b^2(t)/2}$ is increasing in $t$ for $t<t_0$ and decreasing in $t$ for $t>t_0$. 

For $|\xi|\le \sqrt{m_1}b$, Theorem 2.18 of \cite{siegmund1986boundary} provides
\bes
\P\bigg\{\max_{m_0\le t<m_1}\frac{|W(t)|}{t^{1/2}b}\ge 1\bigg|W(m_1) =\xi\bigg\}
\le \frac{\sqrt{m_1}}{\sqrt{m_0}}e^{-b^2/2+\xi^2/(2m)}. 
\ees
Thus, as $\E[\exp(W^2(m)/(2m))I\{|W(m)|\le bm^{1/2}\}]=2b/\sqrt{2\pi}$ for $m>0$, 
\bes
\P\bigg\{\max_{m_0\le t<m_1}\frac{|W(t)|}{t^{1/2}b}\ge 1, \frac{|W(m_1)|}{\sqrt{m_1}b} < 1\bigg\}
\le \frac{\sqrt{m_1}}{\sqrt{m_0}}\sqrt{2/\pi}be^{-b^2/2}. 
\ees
Because $tW(1/t)$ is also a standard Browning motion, 
\bes
\P\bigg\{\max_{m_0\le t<m_1}\frac{|W(t)|}{t^{1/2}b}\ge 1, \frac{|W(m_0)|}{\sqrt{m_0}b} < 1\bigg\}
\le \frac{\sqrt{m_1}}{\sqrt{m_0}}\sqrt{2/\pi}be^{-b^2/2}. 
\ees
Let $\beta > 0$ and define 
$P_-(u) =\P\big\{\max_{e^u/\beta \le t < e^u}|X(t)| \ge 1, |X(e^u)|<1\big\}$ and 
$P_+(v) = \P\big\{\max_{e^v < t \le \beta e^v }|X(t)| \ge 1, |X(e^v)|<1\big\}$. 
For $u< \log t_0 < v$ 
\bes
P_-(u) \le \sqrt{2\beta/\pi}b(e^u)e^{-b^2(e^u)/2},\quad 
P_+(v) \le \sqrt{2\beta/\pi}b(e^v)e^{-b^2(e^v)/2}. 
\ees
Let $u_0=\log t_0$, $u_k=u_0+k\log \beta$ and 
$k_{n,u} = \lfloor (\log n-u_0-u)/\log\beta\rfloor$.
For $0\le u\le u_0$, 
\bes
&& \P\Big\{\max_{0< t \le n} |X(t)| \ge 1\Big\}
\cr &\le& \P\Big\{\max_{t_0 e^{-u}\le t \le t_0 e^{u}} |X(t)| \ge 1\Big\}
+  \sum_{k\le 0}P_-(u_k-u)
+  \sum_{k=0}^{k_{n,u}}P_+(u_k+u).
\ees
Integrating the above inequality over $[0,\log \beta]$, we find that 
\bes
&& (\log\beta) \P\Big\{\sup_{0< t \le n} |X(t)| \ge 1\Big\}
\cr &\le & \int_0^{\log\beta} \P\Big\{\max_{t_0 e^{-u}\le t \le t_0 e^{u}} |X(t)| \ge 1\Big\}du 
+ \int_{-\infty}^{u_0} P_-(u)du+ \int_{u_0}^{\log n} P_+(v)dv 
\cr &\le & \int_0^{\log\beta}\Big\{2\Phi(-b(t_0))+ \min(e^u,2e^{u/2})\sqrt{2/\pi}b(t_0)e^{-b^2(t_0)/2}\Big\}du  
\cr && + \int_{-\infty}^{\log n} \sqrt{2\beta/\pi} b(e^u)e^{-b^2(e^u)/2}du. 
\ees
Let $g_0(\beta) = \min\{\beta-1,3+4(\sqrt{\beta}-2)_+\}$. 
As $\int_0^{\log\beta}\min(e^u,2e^{u/2})du=g_0(\beta)$,
\bes
&& \P\Big\{\sup_{0< t \le n} |X(t)| \ge 1\Big\}
\cr &\le & 2\Phi(-b(t_0))+ \frac{\sqrt{2}g_0(\beta)}{\sqrt{\pi}\log\beta}b(t_0)e^{-b^2(t_0)/2}
+  \frac{\sqrt{2\beta}}{\sqrt{\pi}\log\beta} 
\int_0^n\frac{b(t)e^{-b^2(t)/2}}{t}dt. 
\ees
As $\exp[(b(t)-\sqrt{t})^2/2] = (n/t)\vee e$ and $c_0(n)=n\P\big\{\sup_{0< t \le n} |X(t)| \ge 1\big\}$, 
\bes
c_0(n)
&\le & 2n\Phi(-b(t_0))+ \frac{\sqrt{2}g_0(\beta)}{\sqrt{\pi}\log\beta}nb(t_0)e^{-b^2(t_0)/2}
\cr && +  \frac{\sqrt{2\beta}}{\sqrt{\pi}\log\beta} 
\int_0^n \min(1,n/(te))b(t)e^{(b(t)-\sqrt{t})^2/2-b^2(t)/2}dt. 
\ees
For $\beta = 3.50$ and $n>0$, the right-hand side above is no greater than $1.7068$.

Now consider the case of $n \rightarrow \infty$. Let 
\bes
&& h_n(t) = \frac{n b(t)e^{-b^2(t)/2}}{t^{1/2}e^{-t/2}} = 
\frac{\big(\sqrt{2t\log_+(n/t)} + t\big)
e^{-\sqrt{2t\log_+(n/t)}}}{\max(1,te/n)}. 
\ees
As $h_n(t)\le 1/e+t e^{-\sqrt{2t}}$ and 
$\lim_{n\to\infty} \sup_{t>0} h_n(t)t^{1/2}e^{-t/2} =0$, for large $n$ we have 
$n\Phi(-b(t_0)) \lesssim n b(t_0)e^{-b^2(t_0)/2} \to 0$ and 
\bes
n \int_0^n\frac{b(t)e^{-b^2(t)/2}}{t}dt 
\le \int_0^{\infty} h_n(t)t^{-1/2}e^{-t/2}dt \to 0
\ees
by the dominated convergence theorem. 
Thus, $c(n)\to 0$ as $n\to\infty$. 
\end{proof}


\section{Conclusion}
Our work establishes sharp non-asymptotic regret bounds for UCB indices with a constant level of exploration and a similar non-asymptotic regret bound for Lai's UCB index under the Gaussian reward assumption. In our analysis, the Gaussian assumption can be relaxed to sub-Gaussian assumptions with somewhat messier nonasymptotic regret bounds. Generalization of our analysis to anytime UCB indices is left for future work. Since UCB is widely used in other settings, such as contextual bandits and reinforcement learning, the analytic approach developed in this paper has potential applications beyond the multi-armed bandit problem. 

\section*{Acknowledgements}
We dedicate this article to the memory of Professor Tze Leung Lai. Professor Lai's  seminal contributions to statistics are immensely appreciated by researchers and practitioners alike. His wisdom and insights have inspired and guided us and will continue to do so. 
We would like to thank the editors for giving us the opportunity to write this article to honor him.

\bibliographystyle{apalike}
\bibliography{AMSA_Lai_UCB.bib}



\end{document}